\documentclass[nohyperref]{article}

\usepackage{microtype}
\usepackage{graphicx}
\usepackage{booktabs} 

\usepackage{hyperref}

\usepackage[accepted]{icml2022}

\usepackage{amsmath}
\usepackage{amssymb}
\usepackage{mathtools}
\usepackage{amsthm}
\usepackage[ruled,algo2e]{algorithm2e}

\usepackage[capitalize,noabbrev]{cleveref}
\newcommand{\E}{\mathbb{E}}
\newcommand{\V}{\mathbb{V}}

\theoremstyle{plain}
\newtheorem{lemma}{Lemma}

\newtheorem{fact}{Fact}

\usepackage[textsize=tiny]{todonotes}

\usepackage{multirow}
\usepackage{bbm,bm}
\newcommand{\sm}{\text{Softmax}}

\usepackage{xcolor}
\definecolor{darkred}{RGB}{192,0,0}
\definecolor{oldgold}{RGB}{191,144,0}
\definecolor{darkgreen}{RGB}{84,130,53}

\icmltitlerunning{DGM via Gapped Straight-Through}

\begin{document}

\twocolumn[
\icmltitle{Training Discrete Deep Generative Models via\\ Gapped Straight-Through Estimator
}



\icmlsetsymbol{equal}{*}

\begin{icmlauthorlist}
\icmlauthor{Ting-Han Fan}{equal,pri}
\icmlauthor{Ta-Chung Chi}{equal,cmu}
\icmlauthor{Alexander I. Rudnicky}{cmu}
\icmlauthor{Peter J. Ramadge}{pri}
\end{icmlauthorlist}

\icmlaffiliation{pri}{Department of Electrical and Computer Engineering, Princeton University, Princeton, NJ, USA}
\icmlaffiliation{cmu}{Language Technologies Institute, Carnegie Mellon University, Pittsburgh, PA, USA}

\icmlcorrespondingauthor{Ting-Han Fan}{tinghanf@princeton.edu}
\icmlcorrespondingauthor{Ta-Chung Chi}{tachungc@andrew.cmu.edu}

\icmlkeywords{Machine Learning, ICML}

\vskip 0.3in
]




\printAffiliationsAndNotice{\icmlEqualContribution} 

\begin{abstract}
While deep generative models have succeeded in image processing, natural language processing, and reinforcement learning, training that involves discrete random variables remains challenging due to the high variance of its gradient estimation process. 
Monte Carlo is a common solution used in most variance reduction approaches.
However, this involves time-consuming resampling and multiple function evaluations. 
We propose a Gapped Straight-Through (GST) estimator to reduce the variance without incurring resampling overhead. This estimator is inspired by the essential properties of Straight-Through Gumbel-Softmax. We determine these properties and show via an ablation study that they are essential. Experiments demonstrate that the proposed GST estimator enjoys better performance compared to strong baselines on two discrete deep generative modeling tasks, MNIST-VAE and ListOps.
\end{abstract}

\section{Introduction}
Deep generative models (DGM) \citep{ruthotto2021introduction,kingma2019introduction,goodfellow2020generative,rezende2015variational} are deep neural networks that are capable of high-dimensional probability distributions modeling and random samples generation. These properties are especially useful in applications such as image processing \citep{korshunov2018deepfakes,song2021scorebased}, speech processing \citep{oord2016wavenet}, natural language processing \citep{radford2019language,chen2021evaluating}, and reinforcement learning \citep{ho2016generative,li2017infogail}. Among these tasks, some of which involve inherently discrete components hence necessitating the need of modeling discrete random variables. For example, structure learning \citep{nangia2018listops}, generative text modeling \citep{yang2017improved}, multi-agent control \citep{ryan2017multiagent}, and control with discrete/integer variables \citep{tang2020discretizing,fan2021soft}.
Training these discrete DGMs remains challenging mainly due to the discrete sampling process, which impedes the direct use of gradient backpropagation. Consequently, designing a high-quality gradient estimation technique for the discrete component becomes the key to success.

Existing gradient estimation techniques for discrete DGMs bifurcate into two paradigms: the REINFORCE estimator \citep{glynn1990reinforce,williams1992reinforce} and the Straight-Through Gumbel-Softmax (STGS) \citep{Chris2017gumbel,Jang2017gumbel}. 
The former is unbiased but with high variance, while the latter is of low variance but requires a continuous relaxation during the gradient computation (i.e., $h$ in Eq.~\eqref{eq:st_general}). 
Despite the differences, the Monte Carlo variance reduction technique has been used in both methods; for example, \citet{mnih14nvil,mnih16nvil,gu15muprop} for REINFORCE and \citet{paulus2021raoblackwellizing} for STGS. 
While the variance is reduced, side-effects such as multiple resampling and function evaluations emerge, which are also the main drawbacks that we address.

In this paper, we introduce the Gapped Straight-Through (GST) estimator, a variant of the Straight-Through estimator \citep{bengio2013st}, that adds a careful logit perturbation process. We decide to improve upon the STGS paradigm since it is believed to have low variance and can leave the loss function unmodified (Eq.~\eqref{eq:reinforce-proof} vs. \eqref{eq:repara-proof}). 
First, we show that STGS has a number of properties (\S\ref{sec:property}) essential for good performance (see ablation study in \S\ref{sec:ablation}). To our best knowledge, this is a new determination of the key properties that support STGS. 

\begin{figure*}[ht]
\centering
\includegraphics[width=\textwidth]{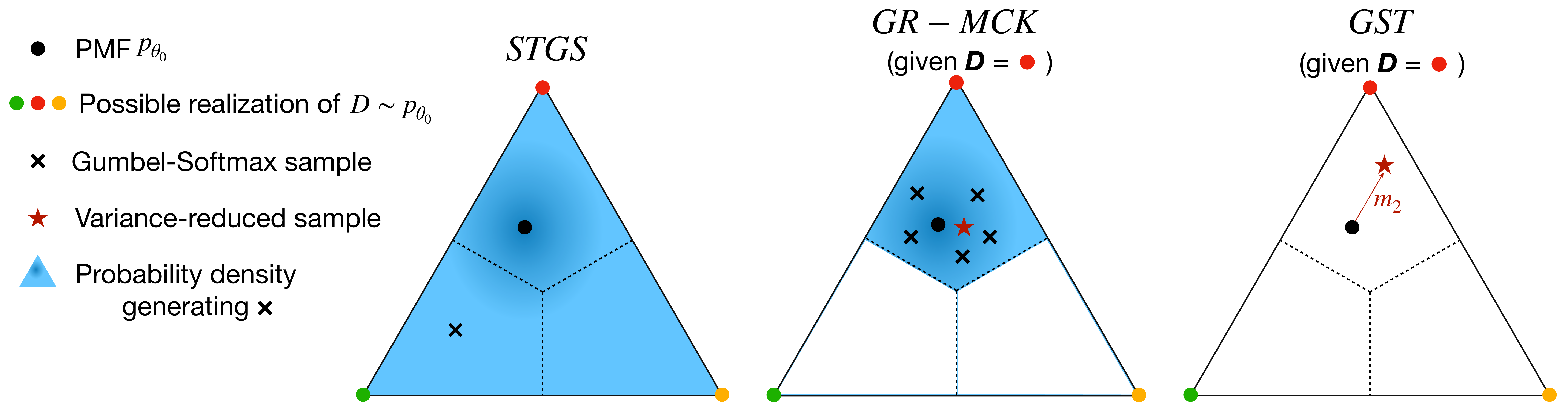}
\caption{\textbf{Different estimators operating on the probability simplex.} The blue shaded region represents the probability density of the Gumbel-Softmax sample in the simplex, where the degree of darkness is proportional to the likelihood. The dashed lines separate three categories (\textcolor{darkred}{top}, \textcolor{darkgreen}{left}, \textcolor{oldgold}{right}). STGS generates soft samples across all categories. GR-MCK \citep{paulus2021raoblackwellizing} chooses a random category $\bm D\sim p_{\theta_0}$ (e.g. \textcolor{darkred}{top} category), generates multiple (e.g. $K=100$) Gumbel-Softmax samples in $\bm D$'s category (i.e., conditioning on $\bm D$), and averages over them to obtain a variance-reduced one. GST chooses a random category $\bm D\sim p_{\theta_0}$ and applies $m_2$ to get a variance-reduced soft sample given $\bm D$ (\S\ref{sec:gst} and Figure~\ref{fig:illustration_gst}). Note that $m_2$ is deterministic given $\bm D$, so the resulting sample is of low variance. All estimators generate the soft samples and convert them into the hard samples using the Straight-Through trick (e.g. Algorithm~\ref{alg:gst}).}
\label{fig:illustration}
\end{figure*}

Second, the Gumbel randomness introduced in STGS is sufficient but not necessary for these properties. We show GST satisfies these properties with less randomness. In particular, GST first samples a random category $\bm D$ and then computes the deterministic perturbation given $\bm D$. This has the advantage of variance reduction while avoiding any resampling overhead, e.g., the Monte Carlo method \citep{paulus2021raoblackwellizing}. 
Experiments show that GST achieves smaller test losses and variances on MNIST-VAE \citep{Jang2017gumbel,kingma2014semi} and better accuracies on ListOps unsupervised parsing \citep{nangia2018listops}. Figure~\ref{fig:illustration} summarizes the difference between our proposed GST and the prior work.

We use bold symbols to denote random variables. For example, $\bm D$, $\bm \xi$, and $\bm G$ are the one-hot random sample, random source, $\text{Gumbel}(0,1)$ random vector, respectively. The subscript $i$, with or without a bracket, denotes the ith entry of a vector; for example, $[p_{\theta}]_i$ and $\bm G_i$ are the ith entry of $p_{\theta}$ and $\bm G$, respectively. Lastly, $\theta$ denotes the trainable parameters of a neural network (NN), and $\text{logit}_{\theta}$ denotes the NN-parameterized vector before the Softmax function.

\section{Gradient Estimation for Discrete DGM}
Let $\bm D$ be a random variable from a discrete/categorical distribution $p_{\theta}$ parameterized by an NN with parameter $\theta$ (we will take $p_{\theta}=\sm_1(\text{logit}_{\theta})$ in \S\ref{sec:stgs}). Without loss of generality, suppose $\bm D$ is one-hot and $\mathbb{P}(\bm D=e_i)=[p_\theta]_i$, where $\{e_1,...,e_N\}$ is the standard basis in $\mathbb{R}^N$. Given an objective function $g:~\{e_1,...,e_N\}\rightarrow \mathbb{R}$, we want to minimize the expected value of $g$ over the distribution of $\bm D$:
\begin{equation}
\min_{p_\theta} \E_{\bm D\sim p_\theta}[g(\bm D)]= \min_\theta \E_{\bm D\sim p_\theta}[g(\bm D)].
\label{eq:main_obj}
\end{equation}
Since $p_\theta$ is parameterized by $\theta$, Eq.~\eqref{eq:main_obj} needs to be optimized by an (unbiased) estimate of $\nabla_\theta \E[g(\bm D)]$, which we will review in this section.

\subsection{REINFORCE Estimator}
The REINFORCE estimator \citep{glynn1990reinforce,williams1992reinforce} is defined as follows:
\begin{equation}
\begin{split}
&\nabla_\theta \E[g(\bm D)]=\sum_{i=1}^N\nabla_\theta [p_\theta]_ig(e_i)\\
&\overset{(*)}{=}\sum_{i=1}^N[p_\theta]_i\frac{\nabla_\theta [p_\theta]_i}{[p_\theta]_i}g(e_i)=\E[\nabla_\theta\log p_\theta(\bm D)g(\bm D)],
\end{split}
\label{eq:reinforce-proof}
\end{equation}
where $p_\theta(\bm D)$ is the probability of $\bm D$; i.e., $p_\theta(\bm D)=[p_\theta]_i$ if $\bm D=i$. 
The LHS of $(*)$ is unweighted while the RHS is written into an expectation with weights $1/[p_\theta]_i$. 
The REINFORCE estimator simply takes Eq.~\eqref{eq:grad-reinforce} as an unbiased estimate of $\nabla_\theta \E[g(\bm D)]$.
\begin{equation}
\nabla_\theta\log p_\theta(\bm D)g(\bm D),~~~\bm D\sim p_\theta.
\label{eq:grad-reinforce}
\end{equation}
While Eq.~\eqref{eq:grad-reinforce} only requires a single sample from $p_\theta$ to establish an unbiased estimate, it is of high variance due to its importance weight $1/[p_\theta]_i$. The variance can be large when there exists small $[p_\theta]_i$ for some $i\in\{1,...,N\}$.

\subsection{Reparameterization and Straight-Through}
The major drawback of REINFORCE is that the randomness ($p_\theta$) is coupled with the NN parameters ($\theta$). This motivates the use of reparameterization to decouple the randomness and the parameter. Suppose that there exists a reparameterization of $\bm D$ (e.g.~Eq.~\eqref{eq:st}) , written as $D(\theta,\bm \xi)$, such that $D(\theta,\bm \xi)\in \{e_1,...,e_N\}$ (i.e., one-hot) and that $D(\theta,\bm \xi)$ is random with the source of randomness from $\bm \xi$. Under such reparameterization, Eq.~\eqref{eq:main_obj} is re-written as
\begin{equation}
\min_\theta \E_{\bm \xi}[g(D(\theta, \bm\xi))].
\label{eq:main_obj_rep}
\end{equation}
The gradient to optimize Eq.~\eqref{eq:main_obj_rep} becomes $\nabla_\theta \E_{\bm \xi}[g(D(\theta,\bm\xi))]$ with $\nabla_\theta g(D(\theta,\bm\xi))$ being its unbiased estimate:
\begin{equation}
\begin{split}
&\nabla_\theta \E_{\bm \xi}[g(D(\theta,\bm\xi))]=\nabla_\theta\int f_{\Xi}(\xi) g(D(\theta,\xi))d\xi\\
&=\int\nabla_\theta f_{\Xi}(\xi) g(D(\theta,\xi))d\xi = \E_{\bm \xi}[\nabla_\theta g(D(\theta,\bm\xi))],
\end{split}
\label{eq:repara-proof}
\end{equation}
where $f_{\Xi}$ is the probability density of $\bm\xi$. Compared with Eq.~\eqref{eq:reinforce-proof}, Eq.~\eqref{eq:repara-proof} does not involve an importance weighting procedure thanks to the decoupling of $\theta$ and $\bm\xi$. Thus, the gradient estimation from the reparameterization (i.e. $D(\theta,\bm\xi)$) is believed to be of lower variance.

The final missing piece is to backpropagate through $D(\theta,\bm \xi)$, which typically requires the differentiability w.r.t. $\theta$. However, the discrete nature of $\bm D$ (i.e., $\bm D \in \{e_1,...,e_N\}$) hinders the differentiability\footnote{We use the concept of~\emph{differentiablility} loosely in this paper. For example, a common misunderstanding is that the $\arg\max$ operation is not differentiable. However, it is differentiable almost everywhere with gradient 0 except for the case when the equality sign holds~\citep{paulus2020gradient}. Nevertheless, we still consider it as~\emph{non-differentiable} in this paper.}.
Fortunately, the Straight-Through estimator \citep{bengio2013st,Chung2017st},~\citep{hinton2012neural}[lecture 15b] helps enforce the discreteness while maintaining differentiability. 
In particular, let $\theta$ be the NN's parameters and $\theta_0=\text{stop\_grad}(\theta)$ be the parameters in the forward pass but with zero gradient in the backward pass (i.e., $\nabla_\theta \text{stop\_grad}(\Phi(\theta))=\nabla_\theta \Phi(\theta_0)=0$ for any $\Phi(\cdot)$). After then, the distribution vector $p_{\theta_0}$ is evaluated, and a one-hot vector $\bm D=D(\theta_0,\bm\xi)$ is sampled from $p_{\theta_0}$. Note $\bm\xi$ is the randomness of the sampling process and the random variable $\bm D$ is now represented by $D(\theta_0,\bm\xi)$. With all the ingredients, the general form of the Straight-Through (ST) estimator can be written as
\begin{equation}
\begin{split}
&D_{\text{ST}}(\theta,\bm\xi)\\ &=\text{stop\_grad}(D(\theta,\bm\xi)) - \text{stop\_grad}(h(\theta,\bm\xi)) + h(\theta,\bm\xi)\\ 
&= D(\theta_0,\bm\xi) - h(\theta_0,\bm\xi) + h(\theta,\bm\xi),
\end{split}
\label{eq:st_general}
\end{equation}
where $h$ is a differentiable function that depends on the NN parameter $\theta$ and (optionally) the randomness $\bm\xi$. Intuitively, $h$ is the surrogate of $\bm D = D(\theta_0,\bm\xi)$ and it allows $D_{\text{ST}}(\theta,\bm\xi)$ to be differentiable. In the forward propagation we have $D_{\text{ST}}(\theta_0,\bm\xi)=\bm D$ while the backward propagation gives $\nabla_\theta D_{\text{ST}}(\theta,\bm\xi)=\nabla_\theta h(\theta,\bm\xi)$. In other words, $D_{\text{ST}}$ has the same value as $\bm D$ and therefore is random and one-hot during the function evaluation but is replaced by $h$ during the gradient computation.
In fact, Eq.~\eqref{eq:st_general} can be viewed as a continuous local approximation around the discrete and non-differentiable $D(\theta_0,\bm\xi)$ by adding a curve $c(\theta)=h(\theta,\bm\xi)-h(\theta_0,\bm\xi)$ such that $c(\theta_0)=0$ and $c(\theta)$ is differentiable w.r.t $\theta$. This builds the local structure around $\theta_0$ and makes the differentiation possible around the neighborhood of $D(\theta_0,\bm\xi)$.

Although Eq.~\eqref{eq:st_general} seems a bit abstract, a naive way to specify $h$ is as follows.
\begin{equation}
D_{\text{ST-naive}}(\theta,\bm\xi) = D(\theta_0,\bm\xi) - p_{\theta_0} + p_{\theta}
\label{eq:st}
\end{equation}
Eq.~\eqref{eq:st} is reduced from Eq.~\eqref{eq:st_general} with $h(\theta,\bm\xi)=p_{\theta}$, where $p_\theta$ is the probability vector modeled by an NN. While being technically feasible, such a naive choice ignores the randomness $\bm\xi$ and is therefore unfavorable. A better choice is to design an $h(\theta,\bm\xi)$ with a strong connection to $\bm D$ considering that the latter is replaced by the former during the gradient computation in Eq.~\eqref{eq:st_general}. We will review a better choice in the next subsection.

\subsection{Straight-Through Gumbel-Softmax (STGS)}
\label{sec:stgs}

Gumbel-Softmax \citep{Jang2017gumbel,Chris2017gumbel} provides a reasonable choice of $h(\theta,\bm\xi)$ that is strongly correlated with $\bm D$ and is differentiable w.r.t the NN parameter $\theta$. The construction is as follows.

Define $\sm_\tau:\mathbb{R}^n\rightarrow\mathbb{R}^n$ as $[\sm_\tau(x)]_i=e^{x_i/\tau}/\sum_{j=1}^ne^{x_j/\tau}$ and let $p_\theta =\sm_1(\text{logit}_\theta)$. That is, the NN parameter $\theta$ generates the unnormalized log probability vector $\text{logit}_\theta$, and $\text{logit}_\theta$ yields the discrete distribution $p_\theta$ through a Softmax function with temperature $\tau=1$. The Gumbel-Max trick \citep{chris2014gumbel} states that
\begin{equation}
\underset{1\leq i\leq N}{\arg\max}~[\text{logit}_\theta+\bm G]_i\sim p_\theta=\sm_1(\text{logit}_\theta),
\label{eq:gumbel-max}
\end{equation}
where $\bm G$ is an i.i.d. $N$-dimensional $\text{Gumbel}(0,1)$ vector. Therefore, we can sample the discrete random variable $\bm D$ using Eq.~\eqref{eq:gumbel-max}. Since the gradient of $\arg\max$ in Eq.~\eqref{eq:gumbel-max} is not useful (either 0 or undefined), Gumbel-Softmax approximates $\arg\max$ of Eq.~\eqref{eq:gumbel-max} using another Softmax:
\begin{equation}
D_{\text{GS}}(\theta,\bm\xi)=\sm_\tau (\text{logit}_\theta+\bm G)
\label{eq:gs}
\end{equation}
After sampling $D(\theta_0,\bm\xi)$ from Eq.~\eqref{eq:gumbel-max} and constructing $h(\theta,\bm\xi)$ from Eq.~\eqref{eq:gs}, we can substitute them into  Eq.~\eqref{eq:st_general} to get the Straight-Through Gumbel-Softmax (STGS) estimator~\citep{Jang2017gumbel}:
\begin{equation}
D_{\text{STGS}}(\theta,\bm\xi)=D(\theta_0,\bm\xi)-D_{\text{GS}}(\theta_0,\bm\xi)+D_{\text{GS}}(\theta,\bm\xi),
\label{eq:stgs}
\end{equation}
where $\bm D=D(\theta_0,\bm\xi)=\text{OneHot}(\underset{1\leq i\leq N}{\arg\max}~[\text{logit}_{\theta_0}+\bm G]_i)$ is the one-hot random sample using the Gumbel-Max in Eq.~\eqref{eq:gumbel-max}. $\bm\xi$ is the randomness from the Gumbel vector $\bm G$. 

Note that Eq.~\eqref{eq:stgs} is reduced from Eq.~\eqref{eq:st_general} with $h(\theta,\bm\xi)=D_{\text{GS}}(\theta,\bm\xi)$. This allows $\bm D$ and $h(\theta,\bm\xi)$ to share the same randomness hence establishes a strong correlation, making Straight-Through Gumbel-Softmax Eq.~\eqref{eq:stgs} preferable to the naive choice in  Eq.~\eqref{eq:st}.

\subsection{Conditional Perspective on STGS}
\label{sec:rao}
Although STGS successfully chooses an $h$ that has a strong correlation with $\bm D$, there is still room for improvement. The key observation is that $\bm D = D(\theta_0,\bm\xi)$ in Eq.~\eqref{eq:stgs} is ``much less random'' than its random source $\bm\xi$. That is, for a fixed discrete value $e_i$, there are multiple instances of $\bm\xi$ that lead to the same $D(\theta_0,\bm\xi)=e_i$. 
From a variance reduction perspective, it is tempting to set $h(\theta,\bm\xi)=h(\theta,D(\theta_0,\bm\xi))=h(\theta,\bm D)$.
In other words, it is enough to make the randomness of $h$ solely come from $\bm D$ if there is a strong correlation between $h(\theta, \cdot)$ and $\bm D$.

Such a variance reduction by conditioning is proposed in \citet{paulus2021raoblackwellizing}. The authors use an averaging over conditional distribution to make the randomness of $h$ almost come from $\bm D$. To be more specific, we can first sample a $\bm D=D(\theta_0,\bm\xi) \sim p_{\theta_0}$ and then sample $\bm J_i(\theta)\overset{\text{i.i.d.}}{\sim}\text{logit}_\theta + \bm G|\bm D$. Finally, a rao-blackwellization scheme can be constructed as follows.
{\small
\begin{equation}
\begin{split}
&D_{\text{GR-MCK}}(\theta,\bm\xi)=\\
&\bm D +\frac{1}{K}\sum_{i=1}^K \left[-\sm_\tau(\bm J_i(\theta_0)) + \sm_\tau(\bm J_i(\theta))\right]\\
&\overset{K\rightarrow\infty}{\rightarrow} \bm D- \E [D_{\text{GS}}(\theta_0,\bm\xi)|\bm D ] + \E [D_{\text{GS}}(\theta,\bm\xi)|\bm D ].
\end{split}
\label{eq:rao}
\end{equation}
}
When $K=1$, Eq.~\eqref{eq:rao} is identical to STGS. Therefore, Eq.~\eqref{eq:rao} implies another construction of STGS: first sample $\bm D\sim p_{\theta_0}$ and then $\text{logit}_\theta + \bm G|\bm D$. Such a two-step process motivates the design of our GST estimator in \S\ref{sec:prop}. On the other hand, when $K$ is large, Eq.~\eqref{eq:rao} converges to the conditional expectation. $\E [D_{\text{GS}}(\theta,\bm\xi)|\bm D ]$'s randomness only depends on $\bm D=D(\theta_0,\bm\xi)$, not $\bm\xi$. Hence, Eq.~\eqref{eq:rao} reduces the randomness through conditioning and averaging.

\section{Key Properties of STGS}
\label{sec:property}
Given the widespread success of STGS, we want to identify its good properties and use them to motivate our estimator design in \S\ref{sec:prop}. An ablation study is conducted in \S\ref{sec:ablation} to verify the usefulness of these properties.

\setcounter{subsection}{-1}
\subsection{Property 0: Following $p_{\theta_0}$}
We want to stress the most basic property of STGS: STGS follows $p_{\theta_0}=\sm_1(\text{logit}_{\theta_0})$ during the forward propagation of a NN. This is because Eq.~\eqref{eq:stgs} reduces to $\bm D$ during the forward pass (i.e. $\theta=\theta_0$), and $\bm D$ is sampled by Gumbel-Max, yielding $\bm D\sim p_{\theta_0}$. As shown in Eq.~\eqref{eq:st_general}, any discrete DGM based on the Straight-Through estimator satisfies this property. Since we focus on the family of STGS, we assume this property holds throughout this paper.

\subsection{Property 1: Consistency}
\label{sec:consistency}
Recall that in Eq.~\eqref{eq:stgs}, $\bm D=\text{OneHot}(\underset{1\leq i\leq N}{\arg\max}~[\text{logit}_{\theta_0}+G]_i)$ is the one-hot sample using the Gumbel-Max trick and $D_{\text{GS}}(\theta_0,\bm\xi)=\sm_\tau (\text{logit}_{\theta_0}+\bm G)$ is its surrogate. Because the $\sm$ operation does not change the relative order of the input vector components, we have: $\underset{1\leq i\leq N}{\arg\max} [\bm D]_i = \underset{1\leq i\leq N}{\arg\max} [D_{\text{GS}}(\theta_0,\bm\xi)]_i$. Thereby, we say a differentiable surrogate function $h$ is consistent with $\bm D$ in Eq.~\eqref{eq:st_general} if:
\begin{equation}
\underset{1\leq i\leq N}{\arg\max} [h(\theta_0,\bm\xi)]_i=\underset{1\leq i\leq N}{\arg\max} [\bm D]_i.
\label{eq:consistency}
\end{equation}
This makes intuitive sense as a surrogate should at least keep the supremacy of the largest component of the input.

\subsection{Property 2: Zero-Gradient Perturbation}
\label{sec:property2}
We now discuss the functional form of the differentiable surrogate, $h(\theta,\bm\xi)$. One observation is that $h$ should sit in the probability simplex $\Delta_{N-1}=\{(u_1,...,u_N):\sum_{i=1}^N u_i=1,~u_i\geq 0~\forall~i\in[1,...,N]\}$, which can be achieved by applying a Softmax or Sparsemax \citep{martins16sparsemax}. This is because $h$ mimics $\bm D$ and $\bm D$ is on the vertices of  $\Delta_{N-1}$. Since we are more compatible with Softmax, we require $h=\sm_\tau(\text{some~logit})$. Then, the design of $h$ boils down to the logit design. Gumbel-Softmax chooses its logit as a perturbed one: $\text{logit}_{\theta}+\bm G$. More generally, we introduce a perturbation function $m$ and write $h$ as 
\begin{equation}
h(\theta,\bm\xi) = \sm_\tau(\text{logit}_{\theta} + m(\theta_0, \bm\xi)).
\label{eq:nondiff}
\end{equation}

Note that $m(\theta_0,\bm \xi)=\bm G$ for Gumbel-Softmax, so $m$ is a perturbation function that generalizes the Gumbel vector. To see why $m$ is designed to be a function of $\theta_0$ and $\bm\xi$, we first note that $m$ has to depend on the sampling randomness $\bm\xi$ so that the surrogate $h$ correlates well with the random one-hot sample $\bm D$. Secondly, the perturbation may depend on $\text{logit}_\theta$, so the dependency on $\theta$ should also be included. Furthermore, $\theta$ is replaced by its zero-gradient version, $\theta_0$, in order to maintain property 1 after a small gradient descent step in $\theta$.\footnote{Let $h'=\sm_\tau(\text{logit}_{\theta} + m(\theta, \bm\xi)).$ After a gradient step, $m(\theta,\bm\xi)$ is changed but $m(\theta_0,\bm \xi)$ is not, so $h$ is more likely to be consistent with $\bm D$ than $h'$ is.} 

Eq.~\eqref{eq:nondiff} brings up the idea of \emph{perturbed logits}. For a zero-gradient perturbation $m(\theta_0,\bm\xi)$, we call $\text{logit}_{\theta}+m(\theta_0,\bm \xi)$ the perturbed logits of $\text{logit}_{\theta}$ by $m$, or simply the perturbed logits. We will discuss the property of perturbed logits in the next subsection, which will be useful for our GST estimator proposed in \S\ref{sec:prop}.

\subsection{Property 3: Strict Gap Between Perturbed Logits}
\label{sec:strict_gap}
In~\S\ref{sec:consistency}, we  focused only on the largest logit, but this leaves us wonder the actual difference between the largest and the other logits. The difference is important since we want to apply the perturbation correctly to reflect the logit difference. To simplify the problem, we study the difference between the top-2 largest perturbed logits of STGS. In particular, conditioning on the event that $i$ is the index of the largest perturbed logit; i.e.,
\begin{equation}
    \bm D=\text{OneHot}(\underset{1\leq j\leq N}{\arg\max}[\text{logit}_{\theta_0}]_j+\bm G_j)=e_i, \nonumber
\end{equation} the expected gap between the top-2 largest perturbed logits is defined as
{\small
\begin{equation}
\begin{split}
&\text{Gap}(\theta_0|\bm D=e_i)= \\
& \E\left[[\text{logit}_{\theta_0}]_i+\bm G_i -\underset{j\neq i}{\max}~  ([\text{logit}_{\theta_0}]_j+\bm G_j)\Big|\bm D=e_i\right].
\end{split}
\label{eq:gap}
\end{equation}
}
Lemma~\ref{lemma:gap} gives an analytical expression of Eq.~\eqref{eq:gap}.
\begin{lemma}
	$\text{Gap}(\theta_0|\bm D=e_i)=-\frac{\log(1-[p_{\theta_0}]_i)}{[p_{\theta_0}]_i}=\frac{\log\left(1+e^{\ell_i-s}\right)}{1-1/(1+e^{\ell_i-s})}$,
	where $\ell_i$ is the shorthand of $[\text{logit}_{\theta_0}]_i$ and $s=\log\left(\sum_{j\neq i} e^{\ell_j} \right)$ is the log-sum-exponential of the unselected logits.
	\label{lemma:gap}
\end{lemma}
A closer look into Lemma~\ref{lemma:gap} shows that the expected gap largely depends on the logit difference $\ell_i-s$ and $s$ is interpreted as \emph{the effective unselected logit}. It is immediate to see that 
\textbf{(a)} the gap increases in $\ell_i-s$,  \textbf{(b)} the gap converges to 1 when $\ell_i-s\ll 0$, and \textbf{(c)} the gap converges to $\ell_i-s$ as $\ell_i-s\gg 0$.

\textbf{(a)} follows from that larger $\ell_i-s$ implies larger logit difference and gap. \textbf{(b)} means the expected gap is strict and is lower-bounded by 1. \textbf{(c)} follows from that the Gumbel noises are negligible when $\ell_i-s$ is large, so the expected gap converges to the effective unperturbed logit difference.

\textbf{(b)} is probably the most counter-intuitive but also the most important property. Although the intuition might suggest a vanishing gap when $\ell_i-s \ll 0$, the event $\bm D=e_i$ turns out to put weight on large enough random gaps such that the expected gap $\geq 1$. 
We provide more discussion on this in Appendix~\ref{appendix:view_to_gap}. The strict expected gap is important for the $\tau$-tempered Gumbel-Softmax, Eq.~\eqref{eq:gs}, to converge to $\bm D$ easier at low temperatures. That is, a strict gap implies $D_{\text{GS}}(\theta_0,\bm\xi)=\sm_\tau(\text{logit}_{\theta_0} +\bm G)\overset{\tau\rightarrow 0}{\rightarrow} D(\theta_0,\bm\xi)=\bm D$ and justifies $D_{\text{GS}}$ as a differentiable surrogate of $\bm D$.

\section{The Proposed Method}
\label{sec:prop}
\subsection{Near-deterministic Straight-Through Estimator}
As discussed in \S\ref{sec:rao}, \citet{paulus2021raoblackwellizing} propose conditional averaging for variance reduction. This is not very efficient in large-scale modeling as the averaging causes an extra computation of size $O(K)$ with $K$ typically being 100 or higher. 
Nevertheless, in Eq.~\eqref{eq:rao}, the average converges to $\E [D_{\text{GS}}(\theta,\bm\xi)|\bm D ]$ at large $K$, which is a deterministic function in $\bm D$. This implies a good choice of deterministic function in $\bm D$, written as $h(\theta,\bm D)$, might improve the performance. 
We thus reduce the general Straight-Through, Eq.~\eqref{eq:st_general} to the following near-deterministic Straight-Through estimator.
\begin{equation}
	D_{\text{ST-det}}(\theta,\bm D) = \bm D - h(\theta_0,\bm D) + h(\theta,\bm D).
	\label{eq:st_det}
\end{equation}
$\bm D=D(\theta_0,\bm\xi)$ is generated from the random source $\bm \xi$. We see that Eq.~\eqref{eq:rao} with $K\rightarrow\infty$ becomes a special case of Eq.~\eqref{eq:st_det} with $h(\theta,\bm D)=\E [D_{\text{GS}}(\theta,\bm\xi)|\bm D ]$. In the next subsection, we will find another $h(\theta,\bm D)$ that also has a high correlation with $\bm D$ but without high resampling cost.

\subsection{Gapped Straight-Through Estimator}
\label{sec:gst}

\begin{figure*}[ht]
\centering
\includegraphics[width=\textwidth]{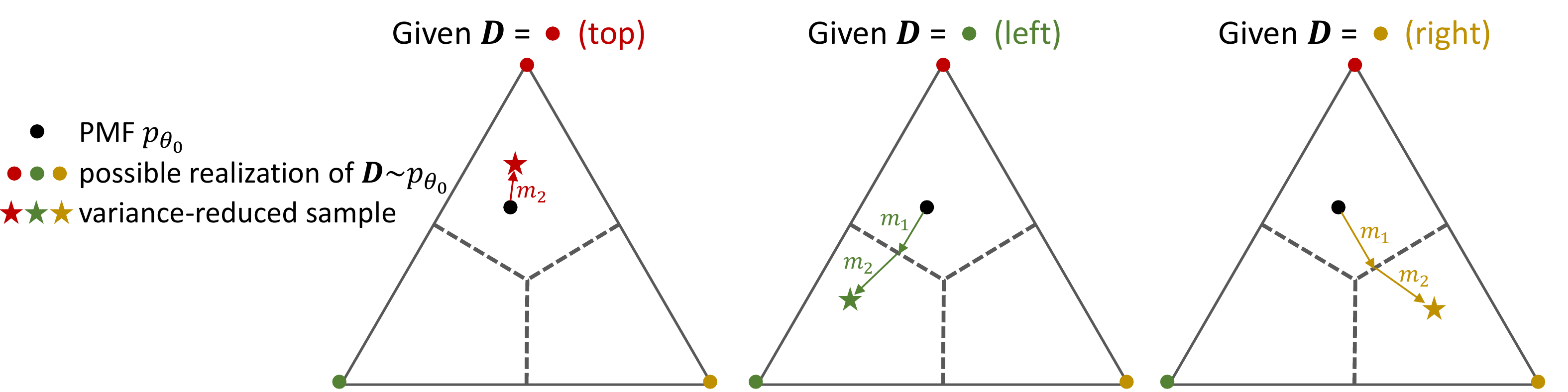}
\caption{\textbf{The soft samples GST estimator on the probability simplex.} The dashed line separate three categories (\textcolor{darkred}{top}, \textcolor{darkgreen}{left}, \textcolor{oldgold}{right}). First, GST chooses a random category $\bm D\sim p_{\theta_0}$, arriving at either the \textcolor{darkred}{red}, \textcolor{darkgreen}{green} or \textcolor{oldgold}{yellow} point. Then, starting from $p_{\theta_0}$, apply $m_1$ and $m_2$ to get a variance-reduced soft sample given $\bm D$. $m_1$ pushes the point to the boundary (if $\arg\max_j~[p_{\theta_0}]_j \neq \arg\max_j~ [\bm D]_j$), and $m_2$ encourages a strict gap between the $\bm D$-selected logit and the unselected ones so that $\bm D$'s category stands out.}
\label{fig:illustration_gst}
\end{figure*}

We now present our Gapped Straight-through Estimator (GST). Due to the success of variance-reduction-type estimator \citep{paulus2021raoblackwellizing}, it is enough to find a good deterministic function, $h(\theta,\bm D)$, for the near-deterministic Straight-Through estimator, Eq.~\eqref{eq:st_det}. Motivated by the success of Gumbel-Softmax, we choose $h(\theta,\bm D)$ based on the properties discussed in \S\ref{sec:property}.

According to property 1 \& 2, we require $h(\theta,\bm D)$ to satisfy
\begin{align*}
&\underset{1\leq j\leq N}{\arg\max}~ [h(\theta,\bm D)]_j=\underset{1\leq j\leq N}{\arg\max}~[\bm D]_j \\
&h(\theta, \bm D) = \sm_\tau(\text{logit}_\theta + m(\theta_0, \bm D)).
\end{align*}
Since Softmax does not change the relative order of the input components, this can be reduced to 
\begin{equation}
\underset{1\leq j\leq N}{\arg\max}~[\text{logit}_{\theta} + m(\theta_0, \bm D)]_j = \underset{1\leq j\leq N}{\arg\max}~[\bm D]_j.
\label{eq:p1+p2}
\end{equation}
To realize Eq.~\eqref{eq:p1+p2}, we design the perturbation $m_1$ that pushes the $\bm D$-selected logit $\langle\text{logit}_{\theta_0},\bm D\rangle$ to be the same as the largest logit. Note $\langle\cdot,\cdot\rangle$ denotes the inner product. Then, Eq.~\eqref{eq:p1+p2} is satisfied when $m(\theta_0,\bm D)=m_1(\theta_0,\bm D)$:
\begin{equation}
\begin{split}
&m_1(\theta_0,\bm D) = \\
&\left(\underset{1\leq j \leq N}{\max}~[\text{logit}_{\theta_0}]_j - \langle \text{logit}_{\theta_0},\bm D \rangle\right)\cdot \bm D.
\end{split}
\label{eq:m1}
\end{equation}

Although a combination of Eq.~\eqref{eq:nondiff} and Eq.~\eqref{eq:m1} realizes properties 1 \& 2, it cannot guarantee a strict gap between the top-2 perturbed logits (i.e., when $\bm D$ does not select the largest logit, $m_1$ makes the top-2 perturbed logits the same, and the gap size becomes zero). To enforce property 3's strict gap, we may either increase the selected logit or decrease the unselected ones. While both are mathematically correct, experiments show the latter gives stable results.

To enforce a strict gap (property 3 of \S\ref{sec:property}), we choose to make the unselected logits smaller. Let $[m_2]_k$ be the decrease on the unselected logit at index $k$ such that the gap size is at least $g\geq 0$. Then, for all $k$ in the indices of unselected logits, we require the following (Note that the selected logit's value becomes $\underset{1\leq j \leq N}{\max}~[\text{logit}_{\theta_0}]_j$ after applying $m_1$.)
$$\underset{1\leq j \leq N}{\max}~[\text{logit}_{\theta_0}]_j-([\text{logit}_{\theta_0}]_k-[m_2]_k)\geq g.$$
To enhance the sparsity, we also require $[m_2]_k\geq0$ where $[m_2]_k=0$ means the kth logit doesn't need the decrease. This avoids the unnecessary perturbation on the logit and makes it simpler to construct $m_2$. $m_2(\theta_0,\bm D,g)$ is defined as follows.
\begin{equation}
\begin{split}
&m_2(\theta_0,\bm D,g)=\\
&\left(\text{logit}_{\theta_0}+g -\underset{1\leq j \leq N}{\max}~[\text{logit}_{\theta_0}]_j\right)_+\cdot (1- \bm D),
\end{split}
\label{eq:m2}
\end{equation}
where $(x)_+=\max(x,0)$ and the product with $(1- \bm D)$ is element-wise (Hadamard product). Combining $m_1$ and $m_2$, we propose the Gapped Straight-Through estimator as
{\small
\begin{equation}
\begin{split}
&D_{\text{GST}}(\theta,\bm D) = \bm D- h(\theta_0,\bm D) + h(\theta, \bm D)\\
&h(\theta, \bm D) = \sm_\tau(\text{logit}_{\theta}+m_1(\theta_0,\bm D) -m_2(\theta_0,\bm D,g) ).
\end{split}
\label{eq:gst}
\end{equation}
}
One might expect $g\approx 1$ due to the limiting behavior of Gumbel-Softmax's expected gap discussed in \S\ref{sec:property}. Furthermore, according to Lemma~\ref{lemma:gap}, we may set the gap as $g=-\frac{\log(1-[p_{\theta_0}]_i)}{[p_{\theta_0}]_i}$ for $\bm D=e_i$. The experiments in \S\ref{sec:experi} show that both produce very similar results.

\subsection{Connection To~\citet{paulus2021raoblackwellizing}}
Let $\E[X], \V[X]$ be the mean and variance of $X$. By the chain rule and the law of total variance, the variances of gradients under GR-MCK and our proposed GST are
\begin{small}
\begin{align*}
	&\V[\nabla_{\text{GR-MCK}}]=\E[\V[\nabla_{\text{GR-MCK}}|\bm D]] + \V[\E[\nabla_{\text{GR-MCK}}|\bm D]]\\
	&~~~~~~~~~~~~~=\underbrace{\frac{1}{K}\E\Big[ \Big(\frac{\partial g(\bm D)}{\partial \bm D}\Big)^2\V\Big[\frac{\partial D_{GS}(\theta,\bm\xi)}{\partial \theta}\Big|\bm D\Big]\Big]}_{(a)}  \\
	&\quad\quad\quad\quad\quad\quad\quad\quad\quad\quad\quad+ \underbrace{\V\Big[\frac{\partial g(\bm D)}{\partial\bm D} \E\Big[\frac{\partial D_{GS}(\theta,\bm\xi)}{\partial \theta}\Big|\bm D\Big]  \Big]}_{(b)}.\\
	&\V[\nabla_{\text{GST}}]= \V\Big[\frac{\partial g(\bm D)}{\partial\bm D} \frac{\partial D_{\text{GST}}(\theta,\bm D)}{\partial \theta}  \Big].
\end{align*}
\end{small}
When $K=1$, $\V[\nabla_{\text{GR-MCK}}]$ becomes $\V[\nabla_{\text{STGS}}]$. Therefore, GR-MCK reduces the variance of STGS by minimizing term (a), which decreases in $K$. On the other hand, $\V[\nabla_{\text{GST}}]$ is very similar to term (b). Since GST follows the key properties of STGS, we expect $\E\Big[\frac{\partial D_{GS}(\theta,\bm\xi)}{\partial \theta}\Big|\bm D\Big]\approx \frac{\partial D_{\text{GST}}(\theta,\bm D)}{\partial \theta}$, implying that GST may enjoy similar variance reduction as GR-MCK. This is verified through experiments in \S\ref{sec:experi}.

\begin{algorithm}
    \SetKwInOut{Input}{Input}
    \SetKwInOut{Output}{Output}
    \SetKwComment{Comment}{/* }{ */}
    \Input{NN para. $\theta$, temperature $\tau$, mode~$\in$\{hard,soft\}}
    \Output{$D_\text{GST}(\theta, \bm D)$}
    $\bm D$ = sample\_onehot\_from($p_{\theta_0}$); $p_{\theta_0}=\sm_1(\text{logit}_{\theta_0})$.
    
    $m_1$ = Eq.~\eqref{eq:m1}\Comment*{consistency}
    $m_2$ = Eq.~\eqref{eq:m2}\Comment*{strict gap}
    
    $h(\theta, \bm D)$ = $\sm_\tau(\text{logit}_{\theta}+m_1-m_2)$ by Eq.~\eqref{eq:gst};

    \eIf{mode is hard}
      {
      return $\bm D$ - stop\_gradient($h(\theta, \bm D)$) + $h(\theta, \bm D)$ \Comment*{Straight-Through trick for hard samples}
      }
      {
        return $h(\theta, \bm D)$ \Comment*{soft samples}
      }
    \caption{The Proposed GST Estimator}
    \label{alg:gst}
\end{algorithm}

\section{Experiments}
\label{sec:experi}
In this section, we compare our Gapped Straight-Through (GST) estimator with the Straight-Through-based model such as STGS \citep{Jang2017gumbel} and its Monte Carlo variance reduction variant, GR-MC100 \citep{paulus2021raoblackwellizing}. Note MC100 means taking 100 independent samples during the Monte Carlo sampling. All models generate~\emph{hard samples} using the Straight-Through trick. \emph{Soft samples} are the outputs of their differentiable surrogate functions, e.g. Eq.~\eqref{eq:gs} and $h(\theta,\bm D)$ in Eq.~\eqref{eq:gst}.

We evaluate the models on two standard tasks, MNIST-VAE \citep{Jang2017gumbel,kingma2014semi} and ListOps unsupervised parsing \citep{nangia2018listops}. MNIST-VAE takes each digit of an MNIST image as a random variable and trains a Variational Auto Encoder \citep{kingma2019introduction} to generate images. The latent space is represented by 30 categorical variables, each with 10 categories. i.e., 30 one-hot vectors with a total dimension of 30x10. ListOps is a dataset composed of prefix arithmetic expressions such as $\min(3, \max(5, 6))$, and the objective is to predict their evaluation results. The task is typically addressed by a tree-LSTM \citep{tai15treelstm} that learns a distribution over latent parse trees. The categorical random variable models the distribution of the parent node over the plausible candidates \citep{choi2018learning}. Since the latent space of MNIST-VAE accepts either discrete or continuous representations while ListOps' parent node selection is strictly discrete, the \emph{soft sample} option is only allowed in MNIST-VAE and is forbidden in ListOps.\footnote{We release our code for both tasks at: 
\url{https://github.com/chijames/GST}.}

\subsection{MNIST-VAE}
A variational autoencoder that generates the MNIST images can be trained by maximizing the evidence lower bound (ELBO) on the log-likelihood:
{\small
\begin{equation*}
\begin{split}
&\log L_{\theta,\phi}(x)\\
&\geq \E_{\bm D^{1:30}\sim p_{\theta}(\cdot|x)}\left[\log \frac{L_{\phi}(x|\bm D^{1:30})\text{Pri}(\bm D^{1:30})}{p_{\theta}(\bm D^{1:30}|x)} \right]\\
&= \E_{\bm D^{1:30}\sim p_{\theta}(\cdot|x)} \log L_{\phi}(x|\bm D^{1:30})-KL(p_{\theta}(\cdot|x)\|\text{Pri}(\cdot))\\
&=\text{ELBO}_{\theta,\phi}(x).
\end{split}
\end{equation*}
}
$x$ is an input image. $\bm D^{1:30}=[\bm D^1,...,\bm D^{30}]$ is the concatenation of 30 categorical random variables, with each $\bm D^i$ having 10 categories. $p_{\theta}(D^{1:30}|x)$ is the encoder as well as the probability mass of $\bm D^{1:30}$ conditioning on the input image. $L_{\phi}(x|\bm D^{1:30})$ is the decoder that model likelihood of $x$ given $\bm D^{1:30}$. $\text{Pri}(D^{1:30})=1/10^{30}$ is the prior of $\bm D^{1:30}$ and is chosen as the uniform distribution. We train the VAE with the loss function being the negative of ELBO: $$-\E_{x\sim \mathcal{X}_{\text{train}}} \text{ELBO}_{\theta,\phi}(x),~~~\mathcal{X}_{\text{train}}:~\text{training~data}.$$
and test the model by replacing $\mathcal{X}_{\text{train}}$ with the testing data $\mathcal{X}_{\text{test}}$. For simplicity, we train all tasks using the same neural network structure, batch size (=100), epochs (=40), optimizer (Adam, learning rate=0.001) and seeds ($\in$[0,1,...,9]). The only differences are the models (STGS, GR-MC100, GST) and temperatures $\tau\in [1,0.5,0.1]$.

\subsubsection{Ablation Study}
\label{sec:ablation}
To start with, we conduct an ablation study of the properties in \S\ref{sec:property}. We compare five estimators from which only one (GST-1.0) satisfies all properties. Because GST-1.0 performs much better than the others, the study justifies that a reasonable estimator should satisfy all properties at the same time.

Table~\ref{tbl:ablation_settings} summarizes the estimators and their properties. ST is Eq.~\eqref{eq:st}, which lacks the consistency because $\arg\max_i [p_{\theta_0}]_i$ may not be $\arg\max_i \bm D_i$. Gap-0.0 is Eq.~\eqref{eq:gst} with $g=0$ and thus the gap is zero by construction. NZ means "non-zero gradient for $m$"; i.e. $m=m(\theta,\bm \xi)$ and hence $\nabla_{\theta}m\neq 0$.

Table~\ref{tbl:ablation} suggests that as long as an estimator violates any of the properties, its performance will worsen. Note that the best performing one, GST-1.0, satisfies all three properties.

\begin{table}[!ht]
    \centering
    \begin{tabular}{l|lll}
    \hline\hline
    Estimator &  Consistency & $\nabla_\theta m=0$ & $\text{Gap}>0$\\
    \hline
    ST & x & \checkmark & \checkmark\\
    NZ-GST-0.0 & \checkmark & x & x \\
    NZ-GST-1.0 & \checkmark & x & \checkmark \\
    GST-0.0 & \checkmark & \checkmark & x\\
    GST-1.0 & \checkmark & \checkmark & \checkmark \\ 
    \hline\hline
    \end{tabular}
    \caption{\textbf{Estimators and properties 1, 2 \& 3 of section~\ref{sec:property}.}}
    \label{tbl:ablation_settings}
\end{table}


\begin{table}[!ht]
    \centering
    \begin{tabular}{ll|ll}
    \hline\hline
    Estimator & Temperature & Neg. ELBO &  Std.\\
    \hline
    \multirow{3}{*}{ST}
    &1.0 & 123.35 & 0.54\\
    &0.5 & 133.84 & 0.66\\
    \hline
    \multirow{3}{*}{NZ-GST-0.0}
    &1.0 & 128.45 & 0.53\\
    &0.5 & 139.02 & 0.47\\
    \hline
    \multirow{3}{*}{NZ-GST-1.0}
    &1.0 & 205.88 & 0.04\\
    &0.5 & 205.90& 0.05\\
    \hline
    \multirow{3}{*}{GST-0.0}
    &0.5 & 119.10 & 0.29\\
    &1.0 & 115.48 & 0.81\\
    \hline
    \multirow{3}{*}{GST-1.0}
    &1.0 & 113.63 & 1.48\\
    &0.5 & 108.43& 1.08\\
    \hline\hline
    \end{tabular}
    \caption{\textbf{Ablation study on the dev set.} The smaller (average) negative ELBO the better. Std. is the standard deviation.}
    \label{tbl:ablation}
\end{table}

\subsubsection{Comparisons of Different Estimators}
\label{sec:vae-comparison}
We now turn to the comparison of STGS, GR-MC100, and GST. Each of these estimators satisfies the properties in \S\ref{sec:property} and should perform reasonably well. The metrics we focus on are the negative ELBO and the standard deviation.


We can see that the GST estimators outperform all the other estimators at temperatures 1.0 and 0.5. Nevertheless, our GST estimator does not perform well under the low-temperature settings, which surprised us at first. After a closer inspection, we conclude that absolute temperature scales for different estimators are not directly comparable. 
Concretely, we calculate the entropy of the probability distribution calculated in Eq. (\ref{eq:gst}). For the first 100 steps at temperature 0.1, the entropy of GR-MC100, STGS, and GST-1.2 are 0.36, 0.15, 0.0007, respectively. The effect is detrimental to GST-1.2 as it does not have enough exploration during the early stage of training. 
The inferior performance of STGS and GR-MC100 at temperature=0.01 also corroborates the effect. Still, this does not undermine the superior quality of the GST estimator, as we can fix the temperature at 1.0 or 0.5 and outperform the previous state-of-the-art GR-MC100 at its best setting. For interested readers, we offer a mixed temperature training strategy that can stabilize the training process of GST at low temperatures in Appendix~\ref{appendix:more_vae_experi}.

Finally, we notice that GST-pi may not perform better than GST-1.0 or GST-1.2 even though GST-pi follows Lemma~\ref{lemma:gap}: Eq.~\eqref{eq:gst} with the expected gap of STGS, $g=-\frac{\log(1-[p_{\theta_0}]_i)}{[p_{\theta_0}]_i}$. Recall that the expected gap converges to 1 when $\ell_i-s \ll 0$ or equivalently, $[p_{\theta_0}]_i\rightarrow 0$. Hence, the limiting behavior of the expected gap ($g=1$) might be more useful than the exactly expected gap at a certain $[p_{\theta_0}]_i$. It is enough to choose $g\approx 1$ for the MNIST-VAE task.


\begin{table}[!ht]
    \centering
    \begin{tabular}{cl|cc}
    \hline\hline
    Temperature & Estimator & Neg. ELBO &  Std.\\
    \hline
    \multirow{6}{*}{1.0} & STGS & 122.96 & 3.08\\
    &GR-MC100 & 120.65 & 2.95\\
    &GST-0.8 & 113.61 & 1.96\\
    &GST-1.0 & 113.63 & 1.48\\
    &GST-1.2 & \bf 112.58 & \bf 1.11\\
    &GST-pi & 112.72 & 1.63\\
    \hline
    \multirow{6}{*}{0.5} & STGS & 118.96 & 2.51\\
    &GR-MC100 & 117.88 & 3.01\\
    &GST-0.8 & 111.54 & 1.30\\
    &GST-1.0 & 108.43 & 1.08\\
    &GST-1.2 & \bf 107.33 & \bf 0.69\\
    &GST-pi & 109.42 & 1.32\\
    \hline
    \multirow{3}{*}{0.1} & STGS & 127.70 & 3.99\\
    &GR-MC100 & 123.23 & 3.55\\
    &GST-* & X &X\\
    \hline
    \multirow{3}{*}{0.01} & STGS & 141.61 & 3.53\\
    &GR-MC100 & 130.04 & 3.63\\
    &GST-* & X &X\\
    \hline\hline
    \end{tabular}
    \caption{\textbf{Comparison of estimators on MNIST-VAE.} The results are evaluated over ten different random seeds. GST-* denotes the GST-based estimators. X means the estimator cannot converge.}
    \label{tbl:vae-compare}
\end{table}

\subsubsection{Empirical Variance Comparison}
We evaluate the variance of the gradient estimators once at the end of each epoch on the MNIST-VAE task. Figure~\ref{fig:var} shows our GST estimator achieves lower variances compared with STGS and GR-MC100 throughout the entire training process.
\begin{figure}[!ht]
    \centering
    \includegraphics[width=0.4\textwidth]{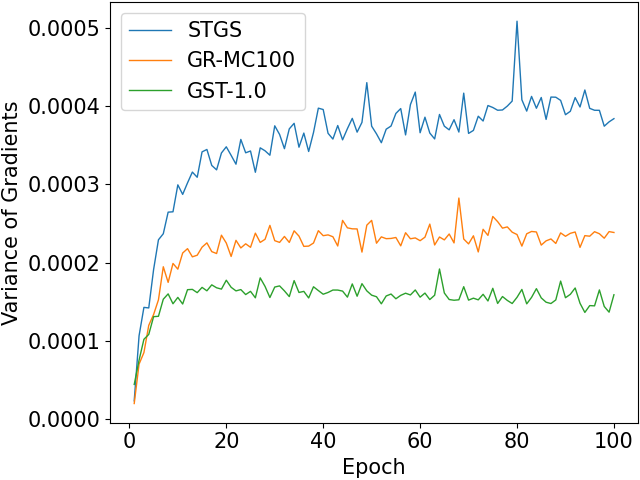}
    \vspace{-3mm}
    \caption{Variance Comparison on MNIST-VAE.}
    \label{fig:var}
\end{figure}

\subsubsection{Comparisons with REINFORCE-based estimators}
To better position our work in the literature of discrete gradient estimation techniques, we additionally compare our estimators, GST-1.0 and GST-1.2, with several state-of-the-art REINFORCE-based strong baselines~\cite{disarm,NEURIPS2021_cd0b43ea}. In Table~\ref{tbl:vae-compare-couple}, we closely follow the network architecture (e.g., sizes, activations) and all the hyperparameters (e.g., batch size, epochs) by carefully inspecting the released code at \url{https://github.com/google-research/google-research/tree/master/disarm}. We can see that our estimator outperforms all the baselines.

\begin{table}[!ht]
    \centering
    \small
    \begin{tabular}{l|ll}
    \hline\hline
    Estimator & Neg. ELBO\\
    \hline
    RLOO & $104.03 \pm 0.23$\\
    DisARM-Tree & $103.10 \pm 0.25$\\
    STGS & 97.32 $\pm$ 0.20\\
    GR-MC100$^*$ & 110.74 $\pm$ 1.23 \\
    GST-1.0 & {\bf 96.09 $\pm$ 0.25} \\
    GST-1.2 &  96.16 $\pm$ 0.32 \\
    \hline\hline
    \end{tabular}
    \caption{\textbf{Comparison of Estimators on MNIST-VAE with~\citet{NEURIPS2021_cd0b43ea}}. The identical implementation (5 random seeds, 200 batch size, $5\times10^5$ training steps,  32 latent variables with 64 categories, optimizer, and architecture) is adopted. $^*$GR-MC100 fails (diverges and yields computational errors) around epoch 40 to 70 for all seeds at all temperatures. Here, we report the numbers at temperature=0.5 right before it fails. We hypothesize that the failure can be attributed to its complicated resampling process.}
    \label{tbl:vae-compare-couple}
\end{table}

\subsubsection{Computational Efficiency}
Table~\ref{tbl:computation} is the computational efficiency comparison for the MNIST-VAE task on one Nvidia 1080-Ti GPU. GST is as efficient as STGS. In contrast, GR-MCK performs a resampling for $K$ times, consuming more memory. For the 1000 sample case, it is not possible for the GPU to parallelize all computation at once, so longer time is needed.
\begin{table}[!ht]
    \small
    \centering
    \begin{tabular}{l|cc}
        Estimator & Allocated GPU Mem. & Time\\ \hline
        STGS & 0.023 GB & 5.03 sec/epoch\\
        GST-1.0 & 0.023 GB & 5.04 sec/epoch\\
        GR-MC100 & 0.081 GB & 5.09 sec/epoch\\
        GR-MC1000 & 0.602 GB & 8.73 sec/epoch
    \end{tabular}
    \caption{Computational Efficiency in Terms of Memory and Speed.}
    \label{tbl:computation}
\end{table}
\vspace{-3mm}

\subsection{ListOps}
For ListOps \citep{nangia2018listops}, we use the same model architecture as in \citet{choi2018learning}. In particular, we want to optimize:
\begin{equation}
    \underset{q_\theta(T|x)}{\E} [\log p_\phi(y|T,x)]
\end{equation}
Both $q_\theta$ and $p_\phi$ are modeled by NNs. Given an arithmetic expression with $M$ tokens, $m=m_1, m_2, \cdots, m_M$, $q_\theta$ is responsible for sampling a parse tree $T$ by merging a pair of adjacent tokens $m_i,m_{i+1}$ at a time. In other words, $m_i$ and $m_{i+1}$ are merged into one new representation, and the total length is reduced by 1. After $M-1$ operations, the expression will end up with only one representation, from which the model predicts the final evaluation results and optimizes the parameters using the cross-entropy loss. As the construction of $T$ requires discrete sampling,  STGS is used with hard samples for training~\citep{choi2018learning,paulus2021raoblackwellizing}.
To validate the effectiveness of our GST estimator, we replace STGS with our estimator and measure the final accuracy of answer prediction. 


\begin{table}[!ht]
    \centering
    \begin{tabular}{cl|ll}
    \hline\hline
    Temperature & Estimator & Accuracy &  Std.\\
    \hline
    \multirow{6}{*}{1.0} & STGS & 0.659 & 0.006\\
    &GR-MC100 & 0.651 & 0.009\\
    &GST-0.8 & 0.660 & 0.010\\
    &GST-1.0 & \bf 0.662 & \bf 0.005\\
    &GST-1.2 & 0.660 & 0.011\\
    &GST-pi & 0.659 & 0.006\\
    \hline
    \multirow{6}{*}{0.1} & STGS & 0.645 & 0.014\\
    &GR-MC100 & 0.637 & 0.049\\
    &GST-0.8 & 0.664 & 0.020\\
    &GST-1.0 & \bf 0.664 & \bf 0.012\\
    &GST-1.2 & 0.660 & 0.018\\
    &GST-pi & 0.659 & 0.015\\
    \hline
    \multirow{6}{*}{0.01} & STGS & 0.479 & 0.258\\
    &GR-MC100 & \bf 0.662 & \bf 0.002\\
    &GST-0.8 & 0.660 & 0.012\\
    &GST-1.0 & 0.661 & 0.004\\
    &GST-1.2 & 0.661 & 0.004\\
    &GST-pi & 0.652 & 0.018\\
    \hline\hline
    \end{tabular}
    \caption{\textbf{Comparison of estimators on ListOps.} We set the maximum sequence length $M$ to 100. We run five different seeds to get the average accuracy and standard deviation.}
    \label{tbl:listops-compare}
\end{table}

Table~\ref{tbl:listops-compare} shows that the estimators perform differently on the ListOps task.
We first note that GST-1.0 performs the best (in both accuracy and variance) among all other variants in the GST family, corroborating our expected gap derivation in \S\ref{sec:strict_gap}. Therefore, we conclude that 1.0 can be chosen as the default gap for GST.
Next, we compare GST-1.0 with STGS, the baseline estimator without an additional variance reduction technique. We can see that GST consistently outperforms STGS across all temperatures with lower variances. The performance difference is especially pronounced when a very low temperature (0.01) is used.
Finally, we compare GST with the strongest baseline, GR-MC100. We achieve modest but better performances on two temperatures, 1.0 and 0.1. When the lowest temperature of 0.01 is used, GST performs on par with GR-MC100.

\section{Conclusion}
We introduce the Gapped Straight-Through (GST) estimator, a gradient estimator for discrete random variables. We derive the properties of Straight-Through Gumbel-Softmax and use them to design the GST estimator. An ablation study shows these properties are essential for good performance. Compared to prior work, the GST estimator enables variance reduction without resampling. Empirically, the GST estimator works with reduced variances and better performance on the MNIST-VAE and ListOps tasks. We hope the GST estimator helps accelerate the training speed of downstream applications while keeping low variance and good performance.

\section*{Acknowledgement}
We thank the anonymous reviewers for their insightful feedback and suggestions.

\bibliography{mybib}

\begin{thebibliography}{36}
\providecommand{\natexlab}[1]{#1}
\providecommand{\url}[1]{\texttt{#1}}
\expandafter\ifx\csname urlstyle\endcsname\relax
  \providecommand{\doi}[1]{doi: #1}\else
  \providecommand{\doi}{doi: \begingroup \urlstyle{rm}\Url}\fi

\bibitem[Abramowitz et~al.(1988)Abramowitz, Stegun, and
  Romer]{abramowitz1988handbook}
Abramowitz, M., Stegun, I.~A., and Romer, R.~H.
\newblock Handbook of mathematical functions with formulas, graphs, and
  mathematical tables, 1988.

\bibitem[Bengio et~al.(2013)Bengio, L{\'e}onard, and Courville]{bengio2013st}
Bengio, Y., L{\'e}onard, N., and Courville, A.
\newblock Estimating or propagating gradients through stochastic neurons for
  conditional computation.
\newblock \emph{arXiv preprint arXiv:1308.3432}, 2013.

\bibitem[Chen et~al.(2021)Chen, Tworek, Jun, Yuan, Pinto, Kaplan, Edwards,
  Burda, Joseph, Brockman, et~al.]{chen2021evaluating}
Chen, M., Tworek, J., Jun, H., Yuan, Q., Pinto, H. P. d.~O., Kaplan, J.,
  Edwards, H., Burda, Y., Joseph, N., Brockman, G., et~al.
\newblock Evaluating large language models trained on code.
\newblock \emph{arXiv preprint arXiv:2107.03374}, 2021.

\bibitem[Choi et~al.(2018)Choi, Yoo, and Lee]{choi2018learning}
Choi, J., Yoo, K.~M., and Lee, S.-g.
\newblock Learning to compose task-specific tree structures.
\newblock In \emph{Thirty-Second AAAI Conference on Artificial Intelligence},
  2018.

\bibitem[Chung et~al.(2017)Chung, Ahn, and Bengio]{Chung2017st}
Chung, J., Ahn, S., and Bengio, Y.
\newblock Hierarchical multiscale recurrent neural networks.
\newblock In \emph{5th International Conference on Learning Representations,
  {ICLR} 2017, Toulon, France, April 24-26, 2017, Conference Track
  Proceedings}. OpenReview.net, 2017.

\bibitem[Dong et~al.(2020)Dong, Mnih, and Tucker]{disarm}
Dong, Z., Mnih, A., and Tucker, G.
\newblock Disarm: An antithetic gradient estimator for binary latent variables.
\newblock In Larochelle, H., Ranzato, M., Hadsell, R., Balcan, M., and Lin, H.
  (eds.), \emph{Advances in Neural Information Processing Systems}, volume~33,
  pp.\  18637--18647. Curran Associates, Inc., 2020.

\bibitem[Dong et~al.(2021)Dong, Mnih, and Tucker]{NEURIPS2021_cd0b43ea}
Dong, Z., Mnih, A., and Tucker, G.
\newblock Coupled gradient estimators for discrete latent variables.
\newblock In Ranzato, M., Beygelzimer, A., Dauphin, Y., Liang, P., and Vaughan,
  J.~W. (eds.), \emph{Advances in Neural Information Processing Systems},
  volume~34, pp.\  24498--24508. Curran Associates, Inc., 2021.

\bibitem[Fan \& Wang(2021)Fan and Wang]{fan2021soft}
Fan, T.-H. and Wang, Y.
\newblock Soft actor-critic with integer actions.
\newblock \emph{arXiv preprint arXiv:2109.08512}, 2021.

\bibitem[Glynn(1990)]{glynn1990reinforce}
Glynn, P.~W.
\newblock Likelihood ratio gradient estimation for stochastic systems.
\newblock \emph{Communications of the ACM}, 33\penalty0 (10):\penalty0 75--84,
  1990.

\bibitem[Goodfellow et~al.(2020)Goodfellow, Pouget-Abadie, Mirza, Xu,
  Warde-Farley, Ozair, Courville, and Bengio]{goodfellow2020generative}
Goodfellow, I., Pouget-Abadie, J., Mirza, M., Xu, B., Warde-Farley, D., Ozair,
  S., Courville, A., and Bengio, Y.
\newblock Generative adversarial networks.
\newblock \emph{Communications of the ACM}, 63\penalty0 (11):\penalty0
  139--144, 2020.

\bibitem[Gu et~al.(2016)Gu, Levine, Sutskever, and Mnih]{gu15muprop}
Gu, S., Levine, S., Sutskever, I., and Mnih, A.
\newblock Muprop: Unbiased backpropagation for stochastic neural networks.
\newblock In Bengio, Y. and LeCun, Y. (eds.), \emph{4th International
  Conference on Learning Representations, {ICLR} 2016, San Juan, Puerto Rico,
  May 2-4, 2016, Conference Track Proceedings}, 2016.

\bibitem[Hinton et~al.(2012)Hinton, Srivastava, and Swersky]{hinton2012neural}
Hinton, G., Srivastava, N., and Swersky, K.
\newblock Neural networks for machine learning.
\newblock \emph{Coursera, video lectures}, 2012.

\bibitem[Ho \& Ermon(2016)Ho and Ermon]{ho2016generative}
Ho, J. and Ermon, S.
\newblock Generative adversarial imitation learning.
\newblock \emph{Advances in neural information processing systems},
  29:\penalty0 4565--4573, 2016.

\bibitem[Jang et~al.(2017)Jang, Gu, and Poole]{Jang2017gumbel}
Jang, E., Gu, S., and Poole, B.
\newblock Categorical reparameterization with gumbel-softmax.
\newblock In \emph{5th International Conference on Learning Representations,
  {ICLR} 2017, Toulon, France, April 24-26, 2017, Conference Track
  Proceedings}. OpenReview.net, 2017.

\bibitem[Kingma \& Welling(2019)Kingma and Welling]{kingma2019introduction}
Kingma, D.~P. and Welling, M.
\newblock An introduction to variational autoencoders.
\newblock \emph{arXiv preprint arXiv:1906.02691}, 2019.

\bibitem[Kingma et~al.(2014)Kingma, Mohamed, Rezende, and
  Welling]{kingma2014semi}
Kingma, D.~P., Mohamed, S., Rezende, D.~J., and Welling, M.
\newblock Semi-supervised learning with deep generative models.
\newblock In \emph{Advances in neural information processing systems}, pp.\
  3581--3589, 2014.

\bibitem[Korshunov \& Marcel(2018)Korshunov and Marcel]{korshunov2018deepfakes}
Korshunov, P. and Marcel, S.
\newblock Deepfakes: a new threat to face recognition? assessment and
  detection.
\newblock \emph{arXiv preprint arXiv:1812.08685}, 2018.

\bibitem[Li et~al.(2017)Li, Song, and Ermon]{li2017infogail}
Li, Y., Song, J., and Ermon, S.
\newblock Infogail: Interpretable imitation learning from visual
  demonstrations.
\newblock In \emph{Proceedings of the 31st International Conference on Neural
  Information Processing Systems}, pp.\  3815--3825, 2017.

\bibitem[Lowe et~al.(2017)Lowe, WU, Tamar, Harb, Pieter~Abbeel, and
  Mordatch]{ryan2017multiagent}
Lowe, R., WU, Y., Tamar, A., Harb, J., Pieter~Abbeel, O., and Mordatch, I.
\newblock Multi-agent actor-critic for mixed cooperative-competitive
  environments.
\newblock In Guyon, I., Luxburg, U.~V., Bengio, S., Wallach, H., Fergus, R.,
  Vishwanathan, S., and Garnett, R. (eds.), \emph{Advances in Neural
  Information Processing Systems}, volume~30. Curran Associates, Inc., 2017.

\bibitem[Maddison et~al.(2014)Maddison, Tarlow, and Minka]{chris2014gumbel}
Maddison, C.~J., Tarlow, D., and Minka, T.
\newblock A* sampling.
\newblock In Ghahramani, Z., Welling, M., Cortes, C., Lawrence, N., and
  Weinberger, K.~Q. (eds.), \emph{Advances in Neural Information Processing
  Systems}, volume~27. Curran Associates, Inc., 2014.

\bibitem[Maddison et~al.(2017)Maddison, Mnih, and Teh]{Chris2017gumbel}
Maddison, C.~J., Mnih, A., and Teh, Y.~W.
\newblock The concrete distribution: {A} continuous relaxation of discrete
  random variables.
\newblock In \emph{5th International Conference on Learning Representations,
  {ICLR} 2017, Toulon, France, April 24-26, 2017, Conference Track
  Proceedings}. OpenReview.net, 2017.

\bibitem[Martins \& Astudillo(2016)Martins and Astudillo]{martins16sparsemax}
Martins, A. F.~T. and Astudillo, R.~F.
\newblock From softmax to sparsemax: A sparse model of attention and
  multi-label classification.
\newblock In \emph{Proceedings of the 33rd International Conference on
  International Conference on Machine Learning - Volume 48}, ICML'16, pp.\
  1614–1623. JMLR.org, 2016.

\bibitem[Mnih \& Gregor(2014)Mnih and Gregor]{mnih14nvil}
Mnih, A. and Gregor, K.
\newblock Neural variational inference and learning in belief networks.
\newblock In Xing, E.~P. and Jebara, T. (eds.), \emph{Proceedings of the 31st
  International Conference on Machine Learning}, volume~32 of \emph{Proceedings
  of Machine Learning Research}, pp.\  1791--1799, Bejing, China, 22--24 Jun
  2014. PMLR.

\bibitem[Mnih \& Rezende(2016)Mnih and Rezende]{mnih16nvil}
Mnih, A. and Rezende, D.~J.
\newblock Variational inference for monte carlo objectives.
\newblock In \emph{Proceedings of the 33rd International Conference on
  International Conference on Machine Learning - Volume 48}, ICML'16, pp.\
  2188–2196. JMLR.org, 2016.

\bibitem[Nangia \& Bowman(2018)Nangia and Bowman]{nangia2018listops}
Nangia, N. and Bowman, S.~R.
\newblock Listops: A diagnostic dataset for latent tree learning.
\newblock \emph{arXiv preprint arXiv:1804.06028}, 2018.

\bibitem[Oord et~al.(2016)Oord, Dieleman, Zen, Simonyan, Vinyals, Graves,
  Kalchbrenner, Senior, and Kavukcuoglu]{oord2016wavenet}
Oord, A. v.~d., Dieleman, S., Zen, H., Simonyan, K., Vinyals, O., Graves, A.,
  Kalchbrenner, N., Senior, A., and Kavukcuoglu, K.
\newblock Wavenet: A generative model for raw audio.
\newblock \emph{arXiv preprint arXiv:1609.03499}, 2016.

\bibitem[Paulus et~al.(2020)Paulus, Choi, Tarlow, Krause, and
  Maddison]{paulus2020gradient}
Paulus, M.~B., Choi, D., Tarlow, D., Krause, A., and Maddison, C.~J.
\newblock Gradient estimation with stochastic softmax tricks.
\newblock \emph{arXiv preprint arXiv:2006.08063}, 2020.

\bibitem[Paulus et~al.(2021)Paulus, Maddison, and
  Krause]{paulus2021raoblackwellizing}
Paulus, M.~B., Maddison, C.~J., and Krause, A.
\newblock Rao-blackwellizing the straight-through gumbel-softmax gradient
  estimator.
\newblock In \emph{International Conference on Learning Representations}, 2021.

\bibitem[Radford et~al.(2019)Radford, Wu, Child, Luan, Amodei, Sutskever,
  et~al.]{radford2019language}
Radford, A., Wu, J., Child, R., Luan, D., Amodei, D., Sutskever, I., et~al.
\newblock Language models are unsupervised multitask learners.
\newblock \emph{OpenAI blog}, 1\penalty0 (8):\penalty0 9, 2019.

\bibitem[Rezende \& Mohamed(2015)Rezende and Mohamed]{rezende2015variational}
Rezende, D. and Mohamed, S.
\newblock Variational inference with normalizing flows.
\newblock In \emph{International conference on machine learning}, pp.\
  1530--1538. PMLR, 2015.

\bibitem[Ruthotto \& Haber(2021)Ruthotto and Haber]{ruthotto2021introduction}
Ruthotto, L. and Haber, E.
\newblock An introduction to deep generative modeling.
\newblock \emph{GAMM-Mitteilungen}, pp.\  e202100008, 2021.

\bibitem[Song et~al.(2021)Song, Sohl-Dickstein, Kingma, Kumar, Ermon, and
  Poole]{song2021scorebased}
Song, Y., Sohl-Dickstein, J., Kingma, D.~P., Kumar, A., Ermon, S., and Poole,
  B.
\newblock Score-based generative modeling through stochastic differential
  equations.
\newblock In \emph{International Conference on Learning Representations}, 2021.

\bibitem[Tai et~al.(2015)Tai, Socher, and Manning]{tai15treelstm}
Tai, K.~S., Socher, R., and Manning, C.~D.
\newblock Improved semantic representations from tree-structured long
  short-term memory networks.
\newblock In \emph{Proceedings of the 53rd Annual Meeting of the Association
  for Computational Linguistics and the 7th International Joint Conference on
  Natural Language Processing (Volume 1: Long Papers)}, pp.\  1556--1566,
  Beijing, China, July 2015. Association for Computational Linguistics.

\bibitem[Tang \& Agrawal(2020)Tang and Agrawal]{tang2020discretizing}
Tang, Y. and Agrawal, S.
\newblock Discretizing continuous action space for on-policy optimization.
\newblock In \emph{Proceedings of the AAAI Conference on Artificial
  Intelligence}, volume~34, pp.\  5981--5988, 2020.

\bibitem[Williams(1992)]{williams1992reinforce}
Williams, R.~J.
\newblock Simple statistical gradient-following algorithms for connectionist
  reinforcement learning.
\newblock \emph{Machine learning}, 8\penalty0 (3):\penalty0 229--256, 1992.

\bibitem[Yang et~al.(2017)Yang, Hu, Salakhutdinov, and
  Berg-Kirkpatrick]{yang2017improved}
Yang, Z., Hu, Z., Salakhutdinov, R., and Berg-Kirkpatrick, T.
\newblock Improved variational autoencoders for text modeling using dilated
  convolutions.
\newblock In \emph{International conference on machine learning}, pp.\
  3881--3890. PMLR, 2017.

\end{thebibliography}
\bibliographystyle{icml2022}

\newpage
\appendix
\onecolumn

\appendix
\setcounter{lemma}{0}
\setcounter{fact}{0}
\section{Appendix}

\subsection{Proof of Lemma 1}
\begin{fact}
Let $\ell = \text{logit}_{\theta_0}$ be the shorthand of the logit vector. The conditional expression of the Gumbel-Softmax's logit is:
\begin{equation}
[\ell + \bm G |\bm D=e_i]_j \overset{d}{=} \begin{cases}
	-\log \frac{\bm E_i}{Z} & \text{if~}j= i\\
	-\log\Big( \frac{\bm E_j}{e^{\ell_j}} + \frac{\bm E_i}{Z} \Big) & \text{o.w.}
	\end{cases},\quad j\in[1,...,n],
	\label{eq:logit-cond}
\end{equation}
$Z=\sum_{j} e^{\ell_j}$ is the partition function. $\bm E_i\overset{\text{i.i.d.}}{\sim}\text{Exp}(1)$ is a standard exponential random variable.
\end{fact}
\begin{proof}
The expression is given in \citet{paulus2021raoblackwellizing}. Here, we provide a proof based on the joint density function. Let $\bm E=[\bm E_1,...,\bm E_N]$ be an N-dimensional standard exponential random vector. Because $\bm G\overset{d}{=}-\log \bm E$, we have 
$$\ell + \bm G  \overset{d}{=} \ell -\log \bm E = -\log \frac{\bm E}{e^{\ell}},$$
where $\bm E/e^{\ell}$ uses elementwise exponential and division. It is enough to focus on $\bm E/e^{\ell}$ conditioning on the ith element being the \emph{smallest} entry; namely,
\begin{equation}
    e_i=\bm D=\underset{j}{\arg\max}~\ell_j+\bm G_j = \underset{j}{\arg\min}~\bm E_j/e^{\ell_j}= \underset{j}{\arg\min}~\bm X_j.
    \label{eq:argmaxmin}
\end{equation}

Since each element $\bm X_j=\bm E_j/e^{\ell_j}$ follows $\text{Exp}(e^{\ell_j})$. Let $\lambda_j = e^{\ell_j}$, the joint density of $\bm X=\bm E/e^{\ell}$ without conditioning is 
$$f(x_1,...,x_N)=\Pi_{j=1}^N \lambda_j e^{-\lambda_jx_j}.$$

Note that $\underset{j}{\arg\min}~\bm E_j/\lambda_j\sim \text{Exp}(\sum_j \lambda_j)$. In order to derive the conditional density, we might want to take out the density of the smallest entry. Namely,

\begin{equation}
\begin{split}
&\prod_j \lambda_j e^{-\lambda_jx_j}\overset{(*)}{=}\lambda_i e^{-\sum_j\lambda_j x_i}\prod_{j\neq i} \underbrace{\lambda_j e^{-\lambda_j(x_j-x_i)}}_{\text{is~a~density~if~}x_j\geq x_i}\\
\overset{(**)}{=}&\sum_i \frac{\lambda_i}{\sum_j\lambda_j}\sum\nolimits_j\lambda_je^{-\sum_j\lambda_j x_i}\prod_{j\neq i}\lambda_j e^{-\lambda_j (x_j-x_i)}\\
\overset{(***)}{=}&\sum_i \mathbb{P}(\bm X_i=\min_j \bm X_j)f(x_i|\bm X_i=\min_j \bm X_j)\prod_{j\neq i}f(x_j|\bm X_j\geq x_i)\\
=&\sum_i \mathbb{P}(\bm X_i=\min_j \bm X_j)f(x_1,...,x_N|\bm X_i=\min_j \bm X_j).
\end{split}
\label{eq:condi-derivation}
\end{equation} 
Note that $(*)$ holds for all $i$ if we ignore making $\lambda_j e^{-\lambda_j(x_j-x_i)}$ a density. Because $\mathbb{P}(\bm X_i=\min_j \bm X_j)=\lambda_i/\sum_j\lambda_j$, $(**)$ rewrite the equation so that the expression is more relevant to $\min_j \bm X_j$ at Eq.~\eqref{eq:argmaxmin}. Finally, $(***)$ simply identifies $f(x_i|\bm X_i=\min_j \bm X_j)=\sum\nolimits_j\lambda_je^{-\sum_j\lambda_j x_i}$ and $f(x_j|\bm X_j\geq x_i)=\lambda_j e^{-\lambda_j(x_j-x_i)}$.

Since Eq.~\eqref{eq:condi-derivation} implies
$$f(x_1,...,x_N|\bm X_i=\min_j \bm X_j)=\sum\nolimits_j\lambda_je^{-\sum_j\lambda_j x_i}\prod_{j\neq i}\lambda_j e^{-\lambda_j (x_j-x_i)}=Ze^{-Z x_i}\prod_{j\neq i}e^{\ell_j} e^{-e^{\ell_j} (x_j-x_i)},$$
the random vector $\bm E/e^{\ell}$ conditioning on $\bm D=e_i$ is
$$[\bm E/e^{\ell}|\bm D=e_i]_j \overset{d}{=}\begin{cases}
\frac{\bm E_i}{Z} & \text{if~}j=i\\
\frac{\bm E_j}{e^j} + \frac{\bm E_i}{Z} & \text{o.w.}
\end{cases}.$$
Taking $-\log(\cdot)$ on this expression concludes the proof.
\end{proof}

\begin{lemma}
	$\text{Gap}(\theta_0|\bm D=e_i)=-\frac{\log(1-[p_{\theta_0}]_i)}{[p_{\theta_0}]_i}=\frac{\log\left(1+e^{\ell_i-s}\right)}{1-1/(1+e^{\ell_i-s})}$,
	where $s=\log\left(\sum_{j\neq i} e^{[\text{logit}_{\theta_0}]_j} \right)$ is the log-sum-exponential of the unselected logits.
\label{lemma:gap_1}
\end{lemma}
\begin{proof}
Write the expected gap using the conditional logit expression Eq.~\eqref{eq:logit-cond}.
\begin{equation}
\begin{split}
\text{Gap}(\theta_0|\bm D=e_i)=&\E\left[~ \underset{j\neq i}{\min}~\ell_i+\bm G_i - (\ell_j+ \bm G_j)\Big|\bm D=e_i\right]\\
\overset{\eqref{eq:logit-cond}}{=}&\E~ \min_{j\neq i}~-\log \left(\frac{\bm E_i}{Z}\right) + \log\left( \frac{\bm E_j}{e^{\ell_j}} + \frac{\bm E_i}{Z} \right)\\
\overset{[p_{\theta_0}]_j=e^{\ell_j}/Z}{=}&\E~\min_{j\neq i}~\log\left(1+\frac{\bm E_j}{\bm E_i[p_{\theta_0}]_j}\right)=\E~\log\left(1+\frac{1}{\bm E_i}\min_{j\neq i}~\frac{\bm E_j}{[p_{\theta_0}]_j}\right)\\
\overset{}{=}&\E~\log\left(1+\frac{\bm E_0}{\bm E_i\sum_{j\neq i}[p_{\theta_0}]_j}\right)=\E~\log\left(1+\frac{\bm E_0}{\bm E_i(1-[p_{\theta_0}]_i)}\right).
\end{split}
\label{eq:gap-main}
\end{equation}
The last line follows from $\frac{\bm E_j}{[p_{\theta_0}]_j}\sim \text{Exp}([p_{\theta_0}]_j)$ and $\min_j \text{Exp}(\lambda_j)\overset{d}{=}\text{Exp}(\sum_j\lambda_j)$. Note $\bm E_0\sim\text{Exp}(1)$ is an exponential random variable that is independent of $\bm E_i$.

Consider the expectation of Eq.~\eqref{eq:gap-main} w.r.t. $\bm E_0$. It is enough to evaluate the following.
\begin{equation}
\begin{split}
\E_{\bm E_0}\Big[\log\Big( \frac{E_0}{r}+1 \Big)\Big]&=\int_0^\infty \log\Big(\frac{t}{r}+1\Big)e^{-t}dt\\
&=e^r\int_0^\infty \log\Big(\frac{t+r}{r}\Big)e^{-(t+r)}dt\\
&=e^r\int_r^\infty \log\Big(\frac{t}{r}\Big)e^{-t}dt\\
&=e^r\bigg(\Big[\log\Big(\frac{t}{r}\Big) e^{-t}\Big]_{t=\infty}^{t=r} + \int_r^\infty \frac{1}{t}e^{-t}dt\bigg)=e^r\Gamma(0,r).
\end{split}
\label{eq:E0}
\end{equation}
The last line uses integration by parts. Notice $\Gamma(0,r)=\int_r^\infty e^{-t}/t$ is the exponential integral as well as a special case of the incomplete gamma function $\Gamma(s,r)$. Since $E$ denotes an exponential random variable, we write the exponential integral as $\Gamma(0,r)$. Thus, Eq.~\eqref{eq:gap-main} becomes
$\E e^{\bm E_i (1-[p_{\theta_0}]_i)}\Gamma(0,\bm E_i(1-[p_{\theta_0}]_i))$, and we turn to another expectation:
\begin{equation}
\begin{split}
\E e^{\bm E_ir}\Gamma(0,\bm E_ir)&=\int_0^\infty e^{rx}\Gamma(0,rx)e^{-x}dx = \int_0^\infty e^{(r-1)x}\int_{rx}^\infty \frac{1}{t}e^{-t}dtdx\\
&=\int_0^\infty\int_0^{t/r} \frac{1}{t}e^{-t} e^{(r-1)x}dxdt=\int_0^\infty \frac{1}{r-1}\frac{1}{t}\left(e^{-t/r}-e^{-t}\right)dt\\
&=\frac{1}{r-1}\lim_{z\rightarrow 0^+}\Gamma(0,z/r)-\Gamma(0,z) = \frac{\log (r)}{r-1}.
\end{split}
\label{eq:Ej}
\end{equation}
The last equality uses the following series expansion \citep{abramowitz1988handbook}[p. 229, 5.1.11] and that $z\rightarrow 0^+$:
$$\Gamma(0,z)=-\gamma-\log(z)-\sum_{k=1}^\infty\frac{(-z)^k}{kk!}=-\gamma-\log(z) + O(z),$$
where $O(z)$ is the Big-O notation when $z\rightarrow 0^+$, indicating the set $\{f:~|f(z)|\leq Mz~\text{as}~z\rightarrow 0^+\}$ for $M>0$.
This implies $\Gamma(0,z/r)-\Gamma(0,z)=\log(r)+O(z)$ and hence the limit of Eq.~\eqref{eq:Ej} goes to $\log(r)/(r-1)$.

Substitute Eq.~\eqref{eq:E0} and \eqref{eq:Ej} into Eq.~\eqref{eq:gap-main}, we arrive at
$$\text{Gap}(\theta_0|\bm D=e_i)=-\frac{\log(1-[p_{\theta_0}]_i)}{[p_{\theta_0}]_i}.$$

Furthermore, let $s=\log\left(\sum_{j\neq i} e^{\ell_j} \right)$ be the log-sum-exponential of the unselected logits. Then, $1-[p_{\theta_0}]_i=e^s/(e^{\ell_i}+e^s) = 1/(1+e^{\ell_i - s})$ and $[p_{\theta_0}]_i=1-1/(1+e^{\ell_i-s})$. So we have an expression in logit difference:
$$\text{Gap}(\theta_0|\bm D=e_i)=\frac{\log\left(1+e^{\ell_i-s}\right)}{1-1/(1+e^{\ell_i-s})}.$$

\end{proof}

\subsection{An Alternative View to Lemma 1}
\label{appendix:view_to_gap}
Here, we provide an alternative view to Lemma~\ref{lemma:gap_1} based on the equivalent unselected logit $s=\log(\sum_{j\neq i} e^{\ell_j})$. The analysis is less complicated than the proof of Lemma~\ref{lemma:gap_1} and should give more intuition. 

Recall that the gap between the top-2 largest logit is $\bm G_i + \ell_i -\max_{j\neq i} (\bm G_j+\ell_j)$ and that $\ell_j + \bm G_j \overset{d}{=} -\log(\bm E_j/e^{\ell_i})$ where $\bm E_j$'s are i.i.d. $\text{Exp}(1)$ variables. Thus, the gap is reorganized as
\begin{equation}
\begin{split}
&\bm G_i + \ell_i -\max_{j\neq i} (\bm G_j+\ell_j)\overset{d}{=}\bm G_i + \ell_i -\max_{j\neq i} \left(-\log\frac{\bm E_j}{e^{\ell_j}}\right)\\
=&\bm G_i + \ell_i - \left(-\log \left(\min_{j\neq i} \frac{\bm E_j}{e^{\ell_j}}\right)\right) \overset{d}{=} \bm G_i + \ell_i - \left(-\log \left(\frac{\bm E_0}{e^s}\right)\right)\\
=& \bm G_i + \ell_i - (\bm G_0 + s).
\end{split}
\label{eq:gap-reorg}
\end{equation} 
The second line uses $\frac{\bm E_j}{e^{\ell_j}}\sim \text{Exp}(e^{\ell_j})$ and $\min_{j\neq i} \text{Exp}(e^{\ell_j})\overset{d}{=}\text{Exp}(e^s)$. $\bm E_0\sim\text{Exp}(1)$ and is independent of $\bm G_i$. The last line defines $\bm G_0=-\log(\bm E_0)\sim \text{Gumbel}(0,1)$ as another gumbel random variable. 

\emph{This means that the smallest gap is equivalent to the perturbed difference between $\ell_i$ and $s=\log(\sum\nolimits_{j\neq i}e^{\ell_j})$, and $s$ becomes the equivalent logit that summarizes all unselected logits.}

Since $\bm G_i,~\bm G_0\overset{i.i.d.}{\sim}\text{Gumbel}(0,1)$, we know $\bm X=\bm G_i-\bm G_0\sim \text{Logistic}(0,1)$ and Eq.~\eqref{eq:gap-reorg} is simply $\bm X-(s-\ell_i)$. Therefore, conditioning on $\bm G_i+\ell_i$ being the largest perturbed logit (i.e., $\bm D=e_i$), the expected gap is 
$$
\text{Gap}(\theta_0|\bm D=e_i)=\E[\bm X-d | \bm X-d\geq 0],~~~~~~d=s-\ell_i.
$$
Note that $\text{Logistic}(0,1)$ has a density function $f(x)=e^{-x}/(1+e^{-x})^2$. We know $\mathbb{P}(\bm X\geq d)=1-1/(1+e^{-d})$. The conditional expectation is evaluated as:
\begin{equation*}
\begin{split}
&\mathbb{P}(\bm X\geq d)\E[\bm X|\bm X\geq d]=\int_d^\infty x \frac{e^{-x}}{(1+e^{-x})^2}dx=\int_d^\infty x \frac{e^{x}}{(1+e^{x})^2}dx\overset{u=e^x}{=}\int_{1+e^d}^\infty \log(u-1)\frac{du}{u^2}\\
=&\frac{d}{1+e^d}+\int_{1+e^d}^\infty \frac{1}{u}\frac{1}{u-1}du=\frac{d}{1+e^d}-\log\frac{e^d}{1+e^d}=d\left(1-\frac{1}{1+e^{-d}}\right)+\log\left(1+e^{-d}\right).
\end{split}
\end{equation*}
Thereby, the expected gap equals to
\begin{equation}
\E[\bm X-d|\bm X\geq d]=d +\frac{\log\left(1+e^{-d}\right)}{1-1/(1+e^{-d})}-d=\frac{\log\left(1+e^{-d}\right)}{1-1/(1+e^{-d})}=\frac{\log\left(1+e^{\ell_i-s}\right)}{1-1/(1+e^{\ell_i-s})},
\label{eq:gap2}
\end{equation}
which is exactly the same as Lemma~\ref{lemma:gap_1}.

Recall the expected gap converges to $\ell_i-s$ when $\ell_i-s\gg 0$ and to $1$ when $\ell_i-s\ll 0$. We can understand this using the property of logistic distribution $\bm X\sim \text{Logistic}(0,1)$.

\begin{enumerate}
	\item[(a)] When $-d=\ell_i-s\gg 0$, $d$ is very negative and the conditional expectation $\E[\bm X-d|\bm X\geq d]$ is almost an unconditioned one, $\E[\bm X-d]$. Because $\E[\bm X]=0$, the expected gap converges to $-d=\ell_i-s$.
	\item [(b)] When $-d=\ell_i-s\ll 0$, $d$ is very positive. Because $\text{Logistic}(0,1)$ has an exponential tail (this is evident from its density function), the conditional expectation $\E[\bm X-d|\bm X\geq d]$ is approximated by the exponential distribution: $\E[\bm Y-d|\bm Y\geq d]$ with $\bm Y\sim \text{Exp}(1)$. Since $\bm Y|\bm Y\geq d \overset{d}{=} \text{Exp}(1)+d$, the expected gap converges to $1$ in this situation.
\end{enumerate} 

\subsection{Additional Experiments for MNIST-VAE at Low Temperatures}
\label{appendix:more_vae_experi}

As shown in section~\ref{sec:vae-comparison}, the estimators are not easy to train at low temperatures. Still, low temperatures are potentially beneficial because the estimators are nearly one-hot, which reduces the functional mismatch between discrete one-hot vectors and the Softmax outputs. Thus, we demonstrate a training technique to assist the training at low temperatures below.
\begin{itemize}
    \item When mod(batch index,$M$)$=0$, train at low temperatures (e.g., temp.=0.1 or 0.01).
    \item When mod(batch index,$M$)$\neq 0$, train at temperature $\text{mid\_temp}$.
    \item Test at low temperatures.
\end{itemize}
$M$ controls the frequency of the low-temperature training step. The mid temperature $\text{mid\_temp}$ is a temperature value at which the estimator is trained without any issue. The presence of $\text{mid\_temp}$ stabilizes the model after the potentially hazardous low-temperature training. We hence call this method the mixed temperature training.

For the experiments in Table~\ref{tbl:vae-compare2}, we select $(M, \text{mid\_temp})$=$(20, 0.5)$. Compared with Table~\ref{tbl:vae-compare}, the mixed temperature training allows both GST and GR-MC100 to properly perform at low temperatures. It also improves the STGS at temperature 0.1 (Avg. Neg. ELBO $\approx 128$ without mixed training and $\approx119$ with mixed training).

\begin{table}[!ht]
    \centering
    \begin{tabular}{cl|ll}
    \hline\hline
    Temperature & Estimator & Neg. ELBO &  Std.\\
    \hline
    \multirow{6}{*}{0.1} & STGS & 119.38 & 2.36\\
    &GR-MC100 & 117.70 & 3.64\\
    &GST-0.8 & 111.33 & 0.95\\
    &GST-1.0 & 107.73 & 1.10\\
    &GST-1.2 & 107.36 & 0.78\\
    &GST-pi & 109.47 & 1.05\\
    \hline
    \multirow{6}{*}{0.01} & STGS & 121.75 & 2.36\\
    &GR-MC100 & 117.69 & 2.13\\
    &GST-0.8 & 110.82 & 1.03\\
    &GST-1.0 & 107.90 & 1.04\\
    &GST-1.2 & 107.86 & 1.06\\
    &GST-pi & 109.08 & 1.26\\
    \hline\hline
    \end{tabular}
    \caption{\textbf{Comparison of estimators on MNIST-VAE with mixed temperature training.} The results are evaluated over ten different random seeds.}
    \label{tbl:vae-compare2}
\end{table}

\end{document}